\documentclass{article}

\usepackage{times}
\usepackage{graphicx} 
\usepackage{subfigure}
\usepackage{amsthm}
\usepackage{natbib}
\usepackage{amssymb}
\usepackage{mathtools}
\usepackage{bigints}
\usepackage{algorithm}
\usepackage{algorithmic}
\usepackage{color}
\usepackage[dvipsnames]{xcolor}
\usepackage{soul}

\usepackage{hyperref}

\usepackage[accepted]{implicitworkshop2017}

\newcommand{\W}[0]{Wasserstein }
\newcommand{\JS}[0]{Jensen-Shannon }
\newcommand{\R}{\mathbb{R}}

\newcommand{\PP}{\mathbb{P}}
\newcommand{\QQ}{\mathbb{Q}}

\newcommand{\pr}[0]{ p_{r} }
\newcommand{\pf}[0]{ q_{\theta} }
\newcommand{\mf}[0]{ \mathcal{M}_{\theta} }
\newcommand{\mr}[0]{ \mathcal{M}_{r} }

\newtheorem{theorem}{Theorem}[section]
\newtheorem{proposition}{Proposition}[section]
\newtheorem{corollary}{Corollary}[section]
\newtheorem{definition}{Definition}[section]

\newcommand{\MCS}{\operatorname{MCS}}
\newcommand{\supp}{\operatorname{supp}}
\newcommand{\Aone}{\operatorname{\textbf{Assumption \ A}}}
\newcommand{\Atwo}{\operatorname{\textbf{Assumption \ B}}}
\newcommand{\Athree}{\operatorname{\textbf{Assumption \ C}}}
\newcommand{\Afour}{\operatorname{\textbf{Assumption \ D}}}
\newcommand{\Aonetwo}{\operatorname{ \textbf{Assumption A} \ and \ \textbf{B}}}
\newcommand{\KL}{\operatorname{KL}}
\newcommand{\JSD}{\operatorname{JSD}}
\newcommand{\Wa}{\operatorname{W}}
\newcommand{\dist}{\operatorname{dist}}

\icmltitlerunning{Implicit Manifold Learning in GANs}

\icmltitlerunning{Implicit Manifold Learning in GANs}

\begin{document} 

\twocolumn[
\icmltitle{Implicit Manifold Learning on Generative Adversarial Networks}



\icmlsetsymbol{equal}{*}

\begin{icmlauthorlist}
\icmlauthor{Kry Yik Chau Lui}{rbc}
\icmlauthor{Yanshuai Cao}{rbc}
\icmlauthor{Maxime Gazeau}{rbc}
\icmlauthor{Kelvin Shuangjian Zhang}{rbc} 
\end{icmlauthorlist}

\icmlaffiliation{rbc}{Borealis AI, Toronto, Canada}

\icmlcorrespondingauthor{Kry Yik Chau Lui}{yikchau.y.lui@rbc.com}

\icmlkeywords{boring formatting information, machine learning, ICML}

\vskip 0.3in
]



\printAffiliationsAndNotice{}  

\begin{abstract} 
This paper raises an implicit manifold learning perspective in Generative Adversarial Networks (GANs), 
by studying how the support of the learned distribution, 
modelled as a submanifold $\mf$, 
perfectly match 
with $\mr$, the support of the real data distribution.
We show that optimizing \JS divergence forces $\mf$ to perfectly match with $\mr$,  
while optimizing \W distance does not.
On the other hand, 
by comparing the gradients of the \JS divergence and the \W distances ($W_1$ and $W_2^2$) in their primal forms,
we conjecture that \W $W_2^2$ may enjoy desirable properties such as reduced mode collapse. 
It is therefore interesting to design new distances that inherit the best from both distances. 
\end{abstract} 

\section{Introduction}
\label{intro}


Unsupervised learning at present is largely about learning a probability distribution of data, either explicity or implicitly.
This is often achieved by parametrizing a probability distribution $\QQ_{\theta}$, that is close to the real data distribution $\PP_r$ in some sense.
The closeness criterion is typically an integral probability metric (e.g. \W distance) or an $f$-divergence (e.g. KL divergence).
Slightly modifying \citet{arjovsky2017towards}'s definition of \textit{perfectly aligned} ( left in figure \ref{fig:compare} ), 
we say two manifolds $\mf$ and $\mr$ are \textit{positively aligned} 
if the set $\mf \cap \mr$ has a positive measure (center in figure \ref{fig:compare}).
\footnote{Intuitively, $\mf$ and $\mr$ are the same on part of the space. }
In the context of generative modeling, 
two properties are desired for the closeness criterion.
First, it should encourage the support of $\QQ_{\theta}$, 
modelled as $\mf$, 
to positively align with $\mr$. 
This is a geometry problem,
and it may be related to sample quality (more realistic generated samples). 
Second, 
it should make $\QQ_{\theta}$ and $\PP_{r}$ probabilistically similar, 
so samples from $\QQ_{\theta}$ reflect the multi-modal nature of $\PP_{r}$.
This is a probability problem, 
and it may be related to sample diversity (less mode dropping). 
The importance of the latter is well recognized \cite{arjovsky2017wasserstein,arora2017generalization}. 
The first geometric property is desired because $\mr$ might encode important constraints satisfied by real data. 
Consider natural images for example, 
samples from a learned distribution $\QQ_{\theta}$ are likely to be sharp looking if they are on $\mf \cap \mr$. 
In practice, $\PP_{r}$ is often supported on a much lower dimensional submanifold $\mathcal{M}_r$. 
For instance, 
the space of celebrity faces is a tiny submanifold in $\mathbb{R}^{3 \times 64 \times 64}$ with potentially very complicated geometry. 
The dimensionality and geometric complexity can make the positive alignment between $\mf$ and $\mr$ very hard. 
If our goal is to generate realistic samples that respect the implicit 
constraints in real data, 
the emphasis of unsupervised learning should not only be learning the probability distribution $\PP_r$ but also the manifold $\mr$. 
In other words, there is an implicit
manifold learning problem embedded in the explicit task of generative model learning. 

Generative Adversarial Networks (GANs) \cite{goodfellow2014generative} is a popular implicit generative model that offers great flexibility on the choice of objective functions.
Extensive research \cite{nowozin2016f, arjovsky2017wasserstein, li2017mmd, bellemare2017cramer, berthelot2017began} 
has been done on GANs loss function to improve training stability and mode collapse.
This paper explores existing loss functions from a different perspective,
namely implicit manifold learning. 
We show that optimizing \W distance does not guarantee positive alignment between $\mf$ and $\mr$, 
while optimizing \JS divergence does.
Furthermore, 
we attempt to clarify geometric and probabilistic properties of the \W $W_1$, $W_2^2$ metrics and \JS divergence, 
by comparing their theoretical gradients. 
We conjecture that $W_2^2$ has richer geometric properties than $W_1$, 
leading to adaptive gradient update and reduced mode collapse.

\begin{figure}[t]
\centering
    \includegraphics[width=1.05\linewidth]{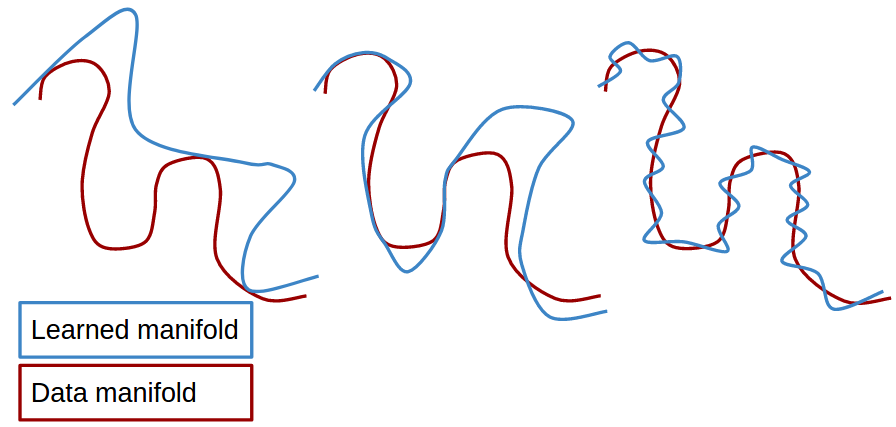}
\caption{Concepts illustrations.
Two manifolds 
\textbf{(left)} perfectly aligned at 3 points;
\textbf{(center)} positively aligned on 3 regions (\JS $\JSD < \log 2$);
\textbf{(right)} intersect transversally at many points (\W $W_p < 0.01 $).
}

\vspace{-0.15in}

\label{fig:compare}
\end{figure}


\section{Preliminaries and Definitions}
Let $\mathcal{X}$ be a compact metric space endowed with Borel $\sigma$-algebra $\Sigma$. 
For a probability measure $\mu$ on $\mathcal{X}$, 
let $\supp(\mu)$ denote its support, 
where $\supp(\mu):=\overline{\{B \in \Sigma | \mu(B)>0\}}$. 
We work with probability distributions whose supports are $k$-dimensional smooth manifolds in the ambient space $\mathbb{R}^n$.
Let $\supp(\PP) = \mathcal{M}_{\PP}^{k} $ and 
$ \supp(\QQ) = \mathcal{M}_{\QQ}^{k} $ (When $k = n$, $\mathcal{M}^{k} = \mathbb{R}^n$).
We focus on two probability distances in this paper, 
the \JS divergence (JSD):
\vspace{-0.09in}
\begin{flalign*}
	\JSD (\PP, \QQ) = \frac{1}{2}\KL(\PP || \QQ_m) + \frac{1}{2}\KL(\QQ|| \QQ_m), 
\end{flalign*}

\vspace{-0.2in} 
where $ \QQ_m = \frac{1}{2}( \PP + \QQ ) $ with $p$, $q$ and $q_m$ denoting densities of
$\PP$, $\QQ$ and $\QQ_m$;

\vspace{-0.07in} 
and Wasserstein $p$-distance $W_p$ ($1\le p < \infty$): 
\vspace{-0.1in}
\begin{flalign}
\label{eqn:Kantorovich}
	\Wa_{p}(\PP, \QQ) 
	= \bigg( \inf\limits_{\gamma \in \Pi(\PP, \QQ)} \int ||x-y||^p d\gamma(x,y)\bigg)^{\frac{1}{p}}
\end{flalign}

\vspace{-0.15in}
where $\Pi(\PP, \QQ)$ denotes the collection of all probability measures on $\mathcal{M_{\PP}}^k \times \mathcal{M_{\QQ}}^k$ with marginals $\PP$ and $\QQ$ on the first and second variables respectively.
As $\mathcal{M_{\PP}}^k $ and $ \mathcal{M_{\QQ}}^k$ have the same dimensions, 
we simplify their notations as $\mathcal{M_{\PP}} $ and $ \mathcal{M_{\QQ}}$ when contexts are clear. 
Monge \cite{monge1781memoire} originally formulated the distance as:
\footnote{ Historically, Monge forumated $W_1$ only. }
\begin{equation}
\label{eqn:Monge}
	W_{p}(\PP, \QQ) 
	= \inf_{T_{*}(\PP) = \QQ} 
	[ \mathbb{E}_{x \sim \pr}  \| x - T(x) \|^{p} ]^{1/p}
\end{equation}

\vspace{-0.15in}

where $T_{*}(\PP) = \QQ$ means a Borel map $T$ pushes forward $\PP$ to $\QQ$, i.e.
$\int_{ T^{-1}(B) } p = \int_{ B } q $
for any Borel set $B \subset \mathcal{M}^{k}$.
Note the infimum in equation \eqref{eqn:Monge} is taken over the space of Borel maps while in equation \eqref{eqn:Kantorovich} the infimum is searched over the space of probability measures.
We consider the cases whenever the infimum is achieved by an optimal transport map $T_p$. 
For example when $p = 2$, 
for each $\QQ$, 
by Brenier's theorem \cite{mccann2001polar, mccann2011five} there exists an optimal transport map $T_2$ such that
$W_{2}^2(\PP , \QQ)
= \mathbb{E}_{x \sim p} [ \| x - T_2(x; \QQ) \|^{2} ] $. 






\section{Sample Quality}
Since its introduction, sample quality in Generative Adversarial Nets (GANs) has improved dramatically \cite{goodfellow2014generative,radford2015unsupervised,berthelot2017began},
and it arguably generates the most realistic looking images nowadays. 
However, little theory exists to explain why this is the case \cite{goodfellow2016nips}. 
One reason is a precise definition of ``sharp looking'' is missing.

When $\PP_{r}$ is the distribution of natural images, 
its support $\supp(\PP_{r})$ is probably sufficiently structured that it can be modeled by a $k$-dimensional submanifold $\mr$ in the ambient space $\mathbb{R}^n$ \cite{narayanan2010sample}. 
Now pick a sample $x$ from $\mr$ and consider its perturbation, $ \widetilde{x} = x + \epsilon$,
where $\epsilon \in \mathbb{R}^n$ and $ \| \epsilon \|$ fixed. 
Depending on $\epsilon$'s direction, 
some $\widetilde{x}$ might look realistic while others may not.
When $ \| \epsilon \|$ increases, the difference becomes more vivid.
This is remarkably similar to the fact that some $\widetilde{x}$ travel along $\mathcal{T}_{x} \mathcal{M}_r $ the tangent space of $\mr$ at $x$ while others go off $\mr$. 
When it is on $\mathcal{T}_{x} \mathcal{M}_r $, $\widetilde{x}$ looks sharper.
When it goes off, $\widetilde{x}$ no longer looks natural. 
This motivates:

\begin{definition}[Realistic Samples]
	
\emph{ 
We say $\QQ$ generates realistic $\PP_{r}$ samples if $ \mathcal{M}_{q} = \supp(\QQ)$ positively aligns with $\mr = \supp(\PP_r)$. 
In other words, samples from $\QQ$ are realistic with respect to $\PP_{r}$ if they lie exactly on $\mr$. }

\end{definition}

In GANs, $\QQ_{\theta}$ is the distribution implicitly parametrized by the generator $G_{\theta}$. 
Ideally, $\QQ_{\theta}$ can generate indistinguishable samples from $\PP_{r}$ after training.
We next show optimizing JSD successfully will necessarily positively align $\mf$ and $\mr$, hence $\QQ_{\theta}$ can generate at least some realistic samples.
This is intuitive, since whenever $\mr$ and $\mf$ do not positively align, 
JSD is maxed out.
We assume the following to translate our intuitions to theorems:

	$\Aone$: $\PP_{r}$ and $\QQ$ are compactly supported on $\mathcal{M}_{r}^k$ and $\mathcal{M}_{q}^k$,
	$k < n$,
	satisfying $\mathcal{L}^k(\mathcal{M}_{r}^k), \mathcal{L}^k(\mathcal{M}_{q}^k)>0$. \footnote{ $\mathcal{L}^k$ denotes Lebesgue measure on $\mathbb{R}^k $. Strictly speaking, $\mathcal{L}^k$ should be replaced by Hausdorff measure $\mathcal{H}^k$. When $k = 2$, $\mathcal{H}^2$ is the measure theoretic surface area. }
 

	$\Atwo$: $\PP_{r}$ and $\QQ$ are absolutely continuous with respect to $\mathcal{L}^k( \mathcal{M}_{r}^k )$ and $\mathcal{L}^k( \mathcal{M}_{q}^k )$, 
	i.e., 
	for any set $B\subset \R^n$, $\PP_{r}(B) = \QQ(B) =0 $ 
	whenever $\mathcal{L}^k(B)=0$. 


\begin{definition}[Minimal common support]
\label{def:MCS}	
\emph{ Under $\Aonetwo$, 
	let $ 0 \leq \alpha \leq  \log 2 $ be given. 
	Consider the set of distributions 
	$\QQ$ that achieve at most $\alpha$ level JSD: 
	$ \Omega^{\alpha} = \{ \QQ: \JSD(\PP_{r}, \QQ) \leq \alpha \} $.}  
	\emph{For any fixed $\PP_{r}$, 
	we define the \textbf{minimal common support} to achieve at most $\alpha$ level JSD to be: 
	$ \MCS^{\alpha}(\PP_{r}) 
	= \inf_{\QQ \in \Omega^{\alpha} } \mathcal{L}^k(\supp(\PP_{r}) \cap \supp(\QQ)  ). $}
\end{definition}


When $\QQ$ is implicitly parametrized by neural networks with parameters $\theta$,
the notations $\Omega^{\alpha}_{\theta}$ and $\MCS^{\alpha}_{\theta}(\PP_{r}) $ 
reflect their dependency on $\theta$.
Definition \ref{def:MCS} captures the worst case scenario: 
when JSD $< \log 2$, 
is $\MCS^{\alpha}(\PP_{r}) > 0$?
In other words, whenever JSD is not maxed out,
can we expect $\QQ$ to generate some $\PP_{r}$ realistic samples with nonzero  probability? 
The next proposition gives a positive answer.


\begin{theorem}[] 
Let $\Aonetwo$ hold and $p_r$, the density of $\PP_{r}$, be bounded, 
then for $\alpha \in [0, \log 2)$, $\MCS^{\alpha} >0$; 
when $\alpha = \log 2$,  $\MCS^{\alpha} = 0$.

\label{thm:well_defined}

\end{theorem}


Theorem \ref{thm:well_defined} ensures $\MCS^{\alpha}(\PP_{r})$ is well-defined. 
The next corollary suggests JSD is a sensible objective to optimize when it comes to generating realistic samples.


\begin{corollary}[]
	Under the assumptions in Proposition \ref{thm:well_defined}, 
	$\MCS^{\alpha}(\PP_{r})$ is non-increasing 
	with respect to $\alpha$ on the interval $[0,\log 2)$.

\label{cor:JS_sharp}

\end{corollary}

The next theorem states optimizing \W distances does not force positive alignment.
In other words, 
there is no guarantee that $\mr$ and $\mathcal{M}_{q} $ positively align unless $W_p(\PP_r, \QQ) = 0$. 
This is because we can find many distributions $\QQ$ such that
$W_p(\PP_r, \QQ) < \epsilon$ but $\mr$ and $\mathcal{M}_{q}$ do not positively align, 
however small $\epsilon > 0 $ gets. 
For pictorial illustrations and comparison of theorems \ref{cor:JS_sharp} and \ref{thm:WS_global}, 
see (center) and (right) in figure \ref{fig:compare}. 


\begin{theorem}

Let $\epsilon > 0$ and $\PP_r$ be a fixed distribution.
Let $ \Gamma 
= \{ \QQ: W_p(\PP_r, \QQ) < \epsilon \}$,  
and consider the decomposition:
$\Gamma = \Gamma_1 \cup \Gamma_2 $, where
$ \Gamma_1 
= \{ \QQ: W_p(\PP_r, \QQ) < \epsilon; 
\mathcal{L}^k (\supp( \PP_r ) \cap \supp( \QQ ) ) > 0 \} $ and 
$ \Gamma_2  = \Gamma - \Gamma_1 $.
Then under $\Aone$,
$\Gamma_2$ is dense in $\Gamma$.

\label{thm:WS_global}

\end{theorem}

As a result, 
the problem that $W_p$ GANs do not necessarily generate realistic samples 
cannot be solved by increasing model capacity.

\section{Sample Diversity and Adaptive Gradient} 


Under finite capacity, 
\cite{arora2017generalization} shows there are mode collapse scenarios that few current training objectives in GANs can prevent. 
In the follow-up empirical analysis, 
\cite{arora2017empirical} raises the open problem on redesigning GANs objective so as to avoid mode collapses. 
A less ambitious quest is to compare the existing loss functions and identify properties related to mode dropping.
Hopefully this suggests new designs that combat mode collapses. 
A natural place to start the comparison is with the gradients of the generator loss functions. 


\subsection{The Wasserstein $W_1$ and $W_2$ distance}
\label{sec:W}
There are empirical evidences showing that \W $W_1$ GANs \cite{arjovsky2017wasserstein} exhibit less mode collapse than \JS GANs.
This is probably due to its geometric properties. 
We attempt to examine this by computing $\nabla_{\theta} W_1( \PP_r, \QQ_{\theta} ) $ in its primal form.  
If the geometric properties of $W_1$ makes it more robust to mode dropping,
then it is also interesting to investigate $W_2$ which better reflects geometric features \cite{villani2008optimal}.
While it is unclear how to apply $W_2$ to GANs training due to its more complex dual formulation, 
it is instructive to analyze its theoretical gradient $\nabla_{\theta} W_2^2( \PP_r, \QQ_{\theta} ) $. 
\vspace{-0.02in}




\begin{proposition}
\label{prp:WS_2}
Let $\PP_r$ and $\QQ_{\theta}$ be two distributions with absolutely continuous densities on $\mr^k$ and $\mf^k$ in the ambient space $\mathbb{R}^n$, with $k \leq n$. 
We have: 
\vspace{-0.1in}
\begin{equation}
	\label{eqn:grad_WS2}
	\nabla_{\theta} W_2^2 ( \PP_{r}, \QQ_{\theta} )
	=  -2\int \left( x - T_2(x; \theta)\right)   \nabla_{\theta} T_2(x; \theta) p_r(x) dx 
\end{equation}
\vspace{-0.07in}
Similarly, we have the following for $W_1 ( \PP_{r}, \QQ_{\theta} )$: 
\begin{equation}
	\label{eqn:grad_WS1}
	\nabla_{\theta} W_1( \PP_{r}, \QQ_{\theta} ) 
	= \int \pm \textbf{1} \nabla_{\theta} T_1(x; \theta) p_r(x) dx
\end{equation}

\vspace{-0.19in}

whenever both sides are well defined. 
$\pm \textbf{1}$ is a vector valued functions with codomain $[ \pm 1, ..., \pm 1]$ 
where the sign depends on whether $(x - T_1(x; \theta))_i$ is positive or negative, 
for $1 \leq i \leq n $.

\end{proposition}


Let us consider the update equation \eqref{eqn:grad_WS2} with one sample point: 
$\theta_{t+1} = \theta_{t} + 2 (x - T_2(x; \theta)) \nabla_{\theta} T_2(x; \theta) $. 
The first term $x - T_2(x; \theta)$ gives $W_2^2$ its geometric properties. 
When $\mr$ and $\mf$ are far away, 
$\| x - T_2(x; \theta) \|$ is very big. 
This should strongly attracts $\mf$ to $\mr$ in $\mathbb{R}^n $.
When $\mr$ and $\mf$ become closer,
$\| x - T_2(x; \theta) \|$ is smaller. 
This resembles $L^2$ optimization in general,
where the loss function offers an adaptive gradient. 
The third term $p_r(x)$ provides a multi-modal weighting. 
The higher $p_r(x)$, 
the stronger contribution it gives to $\nabla_{\theta} W_2^2 ( \PP_{r}, \QQ_{\theta} )$.
Therefore $\PP_r$'s modes will drive the gradient update.

On the other hand, equation \eqref{eqn:grad_WS1} for \W $W_1$ is closer to $L^1$ geometry.
While it has the same probabilistic weighting as $W_2^2$, 
its geometric part is plainer:
the first term is a signed vector $\pm \textbf{1}$ that does not adapt according to $\| x - T_1(x; \theta) \|$ (how far away $\mr$ and $\mf$ are). 
However, 
our analysis is limited because the optimal transport maps $T_1$ and $T_2$ are implicitly defined. 
It is possible that $ \nabla_{\theta} T_1 (q_{\theta})$ and $ \nabla_{\theta} T_2(q_{\theta})$ can cancel the above desired geometric and probabilistic properties. 
Nonetheless, 
we believe the above calculations partially clarify some of the geometric and probabilistic advantages of $W_1$ and $W_2^2$.




\subsection{The Jensen-Shannon Divergence}
In light of previous section,
we perform similar calculations for JSD and the reversed $- \log D$ trick. 
The following assumption is needed to insure KL divergence is finite: 


$\Athree$: Let $\PP_{r}$ and $\QQ_{\theta}$ be absolutely continuous with respect to $\mathcal{L}^n$ with equal support and $\mathcal{L}^n(\supp(\PP_{r}))>0$. 
\footnote{These assumptions make sense when we convolve $\PP_{r} $ and $ \QQ_{\theta}$ with an $n$-dimension Gaussian, as in \cite{arjovsky2017towards}. }




\begin{proposition}\label{prp:negative_logD}
	Let $D^{*}(x)  = \frac{p_{r}(x)}{ q_{\theta_0}(x) + p_{r}(x) }$ be the optimal discriminator, for $\theta_0$ fixed. Under $\Athree$,  
	we have: 

	\vspace{-0.3in}

	\begin{align}
	&\nabla_{\theta} \mathbb{E}_{ z \sim p(z) } \left[ - \log D^{*}( g_{\theta}(z)) \right] |_{\theta = \theta_{0}}\nonumber\\
	&= \mathbb{E}_{\QQ_{\theta}} \left.\left[\nabla_{\theta} \log(q_{\theta}) \left( 
	1 + 
	\log \left( \frac{q_{m}}{p_{r}} \right)\right)  \right]\right|_{\theta = \theta_{0}} 
	\label{eqn:grad_logD},
	\end{align}
	and for the standard JSD:
	\begin{align}
	\nabla_{\theta} &\mathbb{E}_{ z \sim p(z) } [ \log\left(1- D^{*}( g_{\theta}(z)) ]\right)|_{\theta = \theta_{0}}\nonumber\\&
	= \mathbb{E}_{\QQ_{\theta}} \left. \left[\nabla_{\theta} \log( q_{\theta}) \log \left(\frac{ q_{\theta} }{ q_{m} }\right) \right]\right|_{\theta = \theta_{0}} 
	\label{eqn:grad_1_logD}.
	\end{align}
\end{proposition}

\vspace{-0.15in}

Like in section \ref{sec:W}, 
we study the influence of each objective on mode collapse.
We analyze equations \eqref{eqn:grad_logD} and \eqref{eqn:grad_1_logD} where $\pf$ is very small and $\pr$ is comparably large,
which is often the case in early training. 


First we note the influence of $p_r(x)$ is not as obvious as in \eqref{eqn:grad_WS2} or \eqref{eqn:grad_WS1}, as the weight factors $| 1+ \log \frac{q_{m}}{p_{r}} |$ in \eqref{eqn:grad_logD} and $| \log \frac{q_{\theta}}{q_{m}} | $ in \eqref{eqn:grad_1_logD} involve $\pf$ as well. 
Assume $q_{\theta} $ is fixed.
For equation \eqref{eqn:grad_logD} ($-\log D$ trick),
the weight factor $| 1+ \log \frac{q_{m}}{p_{r}} |$ 
strictly decreases as $p_{r}(x)$ gets larger. 
This is undesired because $p_{r}(x)$'s higher probability regions contribute less to $\nabla_{\theta} \log( q_{\theta})$. 
What's worse, 
the regions where $p_{r}(x)$ is small gets a stronger gradient. 
Thus, 
if $q_{\theta}$ misses some modes in the first place,
it may be less likely to learn those modes in later updates. 
In contrast, 
for equation \eqref{eqn:grad_1_logD} (standard \JS GAN), 
the weight $| \log \frac{q_{\theta}}{q_{m}} | $ has the right monotonic relation:
it assigns more weights to regions where $p_r(x)$ is bigger. 
This suggests when $D = D^*$,
the classical $\nabla_{\theta} \JSD( \pr, \pf ) $ is better suited to look for missing modes when the gradient $\nabla_{\theta} q_{\theta} $ does not vanish.
\footnote{In our preliminary experiments, 
when Lipschitz constraints \cite{gulrajani2017improved} is applied to standard JSD GANs, 
$\nabla_{\theta}G_{\theta}(z)$ does not vanish and it trains as well as the $-\log(D)$ trick. 
This is probably due to the preactivation in logit does not lie in the saturation region due to the global Lipschitz constant. }
\footnote{Note this does not necessarily contradict \cite{arjovsky2017towards}'s observation 
that $\nabla_{\theta} \JSD( \PP_r, \QQ_{\theta} ) $ suffers from vanishing gradient. 
Even if $ | \log(\frac{q_{\theta}}{q_m}) | \rightarrow \infty $, so long as $\nabla_{\theta} q_{\theta} \rightarrow 0 $ faster, we still have vanishing gradient. }
Similar to section \ref{sec:W}, 
our analysis is non-conclusive because $ \nabla_{\theta} \log (\pf) $, like $ \nabla_{\theta} T_2(x; \theta) $, is implicitly defined.

\section{Discussions and Future Work}


This paper suggests Wasserstein distances and Jensen-Shannon divergences can complement each other on two important aspects of GANs training,
namely sample quality (sharpness)
and sample diversity (mode collapse). 
Geometric property of \W distance comes from the distance between the \textit{samples} $ \| x - y \|_{x \sim \pr, y \sim \pf } $, 
while \JS divergence acts purely on the densities.
Its sharpness property is due to the logarithmic weights on the densities, i.e. $\log \pr - \log \frac{1}{2}( \pr + \pf) $, 
which heavily penalizes the non-positively aligned supports. 
To preserve both desired properties, 
we can either combine these two measures, 
say by proportional control as in \cite{berthelot2017began} or design a new distance that operates on both samples and the probability densities. 


As the empirical sample quality in \JS GANs does not match our theory, 
identifying the reasons is interesting.  
First, a lower bound of \JS divergence is optimized \cite{nowozin2016f} in practice, 
instead of the divergence itself.
Second, \cite{arora2017generalization} points out the importance of finite sample and finite capacity when we reason GANs training. 
We believe a similar principle applies here. 
Using their definition:


\begin{definition}[$\mathcal{F}$-distance]	
\emph{ Let $\mathcal{F}$ be a class of functions from $\mathbb{R}^n$ to [0, 1]. 
Then $\mathcal{F}$-distance is: 
	\vspace{-0.1in}
	\begin{align*}
		d_{ \mathcal{F}, \log }( \PP, \QQ )
		&= \sup_{D \in \mathcal{F}} | \mathbb{E}_{x \sim \PP } [ \log(D(x)) ] 
		\\&- \mathbb{E}_{x \sim \QQ } [ \log(1 - D(x)) ] | 
		- 2 \log(1/2).
	\end{align*}
}
\end{definition}

\vspace{-0.1in}
When $\mathcal{F}$ 
= \{ all functions from $\mathbb{R}^n$ to [0, 1] \},
$d_{ \mathcal{F} }( \PP, \QQ ) = \JSD (\PP, \QQ) $.
When $\mathcal{F}$ is restricted to a set of neural nets with finite parameters, 
we let $ \widehat{\JSD} $ denote the corresponding neural net distance. 
It is then natural to define a finite capacity version of definition \ref{def:MCS}:


\begin{definition}[Finite Capacity Minimal common support]	
\emph{ 
Let $ 0 \leq \alpha \leq  \log 2 $ be given.
Consider the set of implicitly parametrized distributions $\QQ_{\theta}$ that achieve at most $\alpha$ level JSD:
$ \Omega^{\alpha}_{\theta} 
= \{ \QQ_{\theta}: \widehat{\JSD}( \PP_{r}, \QQ_{\theta}) \leq \alpha \} $. 
For any fixed $\PP_{r}$, 
we define the finite capacity minimal common support to achieve at most $\alpha$ level $\widehat{\JSD}$ divergence to be: 
$ FMCS^{\alpha}_{\theta}(\PP_{r}) 
= \inf_{\QQ_{\theta} \in \Omega^{\alpha}_{\theta} } \mathcal{L}^k(\supp( \PP_{r} ) \cap \supp(\QQ_{\theta}) ). $}
\end{definition}

Under finite capacity and finite sample, 
is it important to understand if a similar conclusion like theorem \ref{thm:well_defined} still holds.
Let $\widehat{\PP_r}$ and $\widehat{\QQ_{\theta}}$ be the corresponding empirical distributions. 
Let $ \widehat{\alpha_1}$ and $ \widehat{\alpha_2}$ be the corresponding $ \widehat{\JSD} $ values computed on finite samples
\footnote{ $ \widehat{\alpha_1} = \widehat{\JSD}( \widehat{\PP_{r}}, \widehat{\QQ_{\theta_1}} ) $ and $ \widehat{\alpha_2} = \widehat{\JSD}( \widehat{\PP_{r}}, \widehat{\QQ_{\theta_2}} ) $, for samples from $\QQ_{\theta_1}$ and $\QQ_{\theta_2}$ }. 
Is is true for sufficiently regular $\mr$ and a moderately sized sample from $\PP_{r}$:
$ \widehat{\alpha_2} < \widehat{\alpha_1} \Rightarrow 
FMCS^{ \widehat{\alpha_2} }_{\theta}(\PP_r) \geq FMCS^{ \widehat{\alpha_1} }_{\theta}(\PP_r) $ with high probability?
\footnote{ The probability is over $\widehat{\QQ_{\theta}}$; we repeatedly sample from $\QQ_{\theta}$. } 
\footnote{ In practice, a sufficiently well trained discriminator $D$ is used to approximate the true neural net distance. } 
More generally, what kind of neural net distance can give the above properties?
Recently, \cite{berthelot2017began} demonstrated impressive sample quality. 
How do their approaches positively align $\mf$ with $\mr$?




Moreover,
since $\mf$ is parametrized by the generator, 
we may regularize $G_{\theta}$ 
based on $\mr$'s geometric structure. 
So the cost functions will include a geometric loss and a probability distance. 

While we discussed implicit manifold learning under GANs framework in this paper,  
it is also interesting to explore this perspective with other generative models such as Variational Autoencoder \cite{kingma2013auto}.

\newpage

\twocolumn[
\icmltitle{Supplementary Materials}]

\section{Proofs}

\begin{proposition}[Proposition \ref{thm:well_defined} in main paper] 
Let $\Aonetwo$ hold and $p_r$, the density of $\PP_{r}$ be bounded, 
then for $\alpha \in [0, \log 2)$, $\MCS^{\alpha} >0$; 
when $\alpha = \log 2$,  $\MCS^{\alpha} =0$.

\end{proposition}

\begin{proof}
 
The proof is divided into two parts. In the first part, we show 
that the minimum common support between $\PP_r$ and $\QQ_{\theta}$
is strictly positive for all $\alpha \in [0, \log 2)$. In the second part, we show the minimum common support is equal to zero for $\alpha = \log 2$.

Let us prove the first part by contradiction and assume that 
there exists an $\alpha_0 \in [0, \log 2)$, such that $\MCS_{\alpha_0} = 0$.
 By definition of the infimum, there exists a minimizing sequence
 of distributions in $\Omega_{\alpha_0}$, denoted as $\{\QQ_{\theta}^m\}_{m=1}^{\infty}$, such that $\mathcal{L}^k(\supp(\PP_{r}) \cap \supp(\QQ_{\theta}^m) ) \rightarrow 0$, as $m\rightarrow \infty$. 
 Then, by definition of the set $\Omega_{\alpha_0}$, $\JSD(\PP_{r}, \QQ_{\theta}^m) \le \alpha_0 <\log 2$ and there is an overlap between $\PP_{r}$ and $\QQ_{\theta}^m$.  

Without loss of generality, we assume $\supp(\QQ_{\theta}^m)\subset \supp(\PP_{r})$. We define the set $S_m: = \supp(\QQ_{\theta}^m)$ and its complementary $S_m^c = \supp(\PP_r)\setminus \supp(\QQ_{\theta}^m)$. 
Moreover for each $J>0$, we define $S_m^J=\{ x\in S_m: q_{\theta}^m(x) \le J  \}$. 

 
We write $2\JSD(\PP_{r}, \QQ_{\theta}^m)= \sum_{j=1}^5 J^m_j$
, where the five terms are given by
 \begin{flalign*}
		J_1^m := &\int_{S_k} p_r(x) \log (2p_r(x)) dx,\\
		J_2^m := &-\int_{S_k} p_r(x)\log (p_r(x)+q_{\theta}^k(x)) dx,\\	
	    J_3^m := & \left. \KL(\QQ_{\theta}^m, (\QQ_{\theta}^m + \PP_r)/2) \right|_{S_m^J},\\
		J_4^m := & \left. \KL(\QQ_{\theta}^m, (\QQ_{\theta}^m + \PP_r)/2) \right|_{S_m \setminus S_{m}^J},\\
		J_5^m:= & \left.2\JSD(\PP_{r}, \QQ_{\theta}^m) \right|_{S_m^c},
	\end{flalign*}

From the inequality  $x\log(2x) \ge -1$ for all $x \ge 0$, we deduce $J_1^m \ge  - \mathcal{L}^n(S_m) $. From the boundedness of $p_r$ by $N$ and using the Jensen inequality applied to the convex function $x\log(x)$, we have $J_2^m \ge   -N\mathcal{L}^k(S_m) \log (N+ 1/\mathcal{L}^k(S_m))$. On the set $S_m^J$, $q_{\theta}^k \le  J $. Therefore from the Jensen inequality, we get $J_3^m  \ge (J+N)\mathcal{L}^k(S_m^J)\cdot \min\limits_{x\ge 0}(x\log x)/2$. By a diagonal extraction argument, we can extract a subsequence $\{q_{\theta}^{m_i} \}_{i=1}^{\infty}$ such that $\mathcal{L}^k(S_{m_i}^i)\le \frac{1}{i^2}$ and $J_4^{m_i} \ge \log(\frac{2i}{i+N})[1-i\mathcal{L}^k(S_{m_i}^i)] \ge \log(\frac{2i}{i+N})(1-\frac{1}{i})$. Finally on $S_k^c$, $q_{\theta}^m= 0$ and as a consequence $J_5^m = \int_{S_{m}^c} p_r(x) \log (2) dx = \log (2) (1-\PP_{r}(S_m))$.

Gathering the above inequalities, we deduce $2\log 2 \ge \lim\limits_{i\rightarrow \infty}2\JSD(\PP_{r}, \QQ_{\theta}^{m_i}) = \lim\limits_{i\rightarrow \infty}J_1^{m_i} + J_2^{m_i} + J_3^{m_i} + J_4^{m_i} + J_5^{m_i} \ge 0+ 0 + 0+ \log 2 + \log 2 = 2\log 2$, as $m_i\rightarrow \infty$. Thus we deduce from the squeeze theorem that there exists a subsequence such that $\lim\limits_{i\rightarrow \infty}\JSD(\PP_{r}, \QQ_{\theta}^{m_i}) = \log 2 > \alpha_0$, which is in contradiction with our assumption that $\alpha_0 < \log(2)$.
	
We now prove the second assertion namely if $\alpha = \log 2$ then the minimum common support is zero. Since $\PP_{r}$ is compactly supported, there exists $x_1\in \R^n$, such that ${\dist}(x_1, \supp(\PP_{r})) >2$. Let $\QQ_{1}$ be a probability distribution on $B_1(x_1)$.	Then, $\JSD(\PP_{r}, \QQ_{1}) = \log 2$ and $\mathcal{L}^k(\supp(\PP_{r})\cap \supp(\QQ_{1})) =0$. Therefore, $\MCS_{\log 2} =0$.
 \end{proof}

\begin{theorem}[Theorem \ref{cor:JS_sharp} in the main paper]
	Under the assumptions in Proposition \ref{thm:well_defined}, 
	$\MCS^{\alpha}(\PP_{r})$ decreases 
	with respect to $\alpha$ on the interval $[0,\log 2)$.

\end{theorem}
 \begin{proof}
 	Let $\alpha < \beta$ be the two $\JSD$ values. 
 	By definition, since $\Omega_{\alpha} \subset \Omega_{\beta}$,
 	we have $\MCS^{\alpha}(p_{r}) \geq \MCS_{\beta}(p_{r})$ automatically. 
\end{proof}

\begin{theorem}[Theorem \ref{thm:WS_global} in the main paper]

Let $\epsilon > 0$ and $\PP_r$ be a fixed distributions.
Moreover let $\Aone$ hold.
Let $ \Gamma 
= \{ \QQ_{\theta}: W_p(\PP_r, \QQ_{\theta}) < \epsilon \}$. 
Consider the decomposition:
$\Gamma = \Gamma_1 \cup \Gamma_2 $, where
$ \Gamma_1 
= \{ \QQ_{\theta}: W_p(\PP_r, \QQ_{\theta}) < \epsilon; 
\mu(\supp( \PP_r ) \cap \supp( \QQ_{\theta} ) ) > 0 \} $ and 
$ \Gamma_2  = \Gamma - \Gamma_1 $.
Then $\Gamma_2$ is dense in $\Gamma$.

\end{theorem}

\begin{proof}
Let $q_{\theta_0} \in \Gamma_1 $
and $\delta = (\epsilon - W_p(\PP_r, \QQ_{\theta_0}))/10 $.
By general position lemma \cite{guillemin2010differential}, for almost every $t \in \mathbb{R}^{n}$,  
$\mf + t $ intersects $\mr$ transversally. 
In particular, for almost every
\footnote{ Almost every with respect to Lebesgue measure $\mathcal{L}^{n}$. }
$t_b \in B_{\delta}^{n}(0)$, $\mathcal{M}_{\theta_0} + t_b $ intersects $\mr$ transversally. 
The new probability measure $\QQ_{\theta_0} + t_b$ is identical to $\QQ_{\theta_0}$ except that its support is translated by $t_b$. 
The difference lies in the fact that the common support of the new measure $\QQ_{\theta_0} + t_b$ and $\PP_r$ has measure zero. 
This translation only affects $\mathcal{M}_{\theta_0}$ by $\delta$, so by definition of $\delta$
\begin{align*}
W_p(\PP_r, \QQ_{\theta_0} + t_b) &< W_p(\PP_r, \QQ_{\theta_0}) + \delta    < \epsilon 
\end{align*}
by recalling definition of Wasserstein distance.
Since we can make $\delta$ arbitrarily small, 
we have shown for every $q_{\theta} \in \Gamma_1 $, we can find another $q_{\theta_0} + t_b \in \Gamma_2 $ that is as close as we like. 
This proves the desired claim.
\end{proof}

\begin{proposition}[Proposition \ref{prp:negative_logD} in the main paper]
	Let $D^{*}(x)  = \frac{p_{r}(x)}{ q_{\theta_0}(x) + p_{r}(x) }$ be the optimal discriminator, for $\theta_0$ fixed. Under $\Athree$ and $\Afour$,  
	we have: 
	\begin{align}
	&\nabla_{\theta} \mathbb{E}_{ z \sim p(z) } \left[ - \log D^{*}( g_{\theta}(z)) \right] |_{\theta = \theta_{0}}\nonumber\\
	&=  2 \left. \nabla_{\theta} \KL(\QQ_m || \PP_r) \right|_{\theta = \theta_{0}}\nonumber\\
	&= \mathbb{E}_{\QQ_{\theta}} \left.\left[\nabla_{\theta} \log(q_{\theta}) \left( 
	1 + 
	\log \left( \frac{q_{m}}{p_{r}} \right)\right)  \right]\right|_{\theta = \theta_{0}},
	\end{align}

	and for the standard JSD: 
	\begin{align}
	\nabla_{\theta} &\mathbb{E}_{ z \sim p(z) } [ \log\left(1- D^{*}( g_{\theta}(z)) ]\right)|_{\theta = \theta_{0}}\nonumber\\&
	= \mathbb{E}_{\QQ_{\theta}} \left. \left[\nabla_{\theta} \log( q_{\theta}) \log \left(\frac{ q_{\theta} }{ q_{m} }\right) \right]\right|_{\theta = \theta_{0}}.
	\end{align}
\end{proposition}

 \begin{proof} 
 	It is known from \citet{arjovsky2017towards} that 
 	\begin{align*}
 	&\mathbb{E}_{ z \sim p(z) } \left[ - \nabla_{\theta} \log D^{*}( g_{\theta}(z))|_{\theta = \theta_{0}} \right] \\
 	&= \nabla_{\theta} \left[  \KL(\QQ_{\theta}|| \PP_r) - 2 \JSD(\PP_r, \QQ_{\theta})\right]|_{\theta = \theta_{0}}
 	\end{align*}
 By definition of the Kullback Leibler divergence, Jensen Shannon distance and from $\Aone$
 	\begin{align*}
 	& \KL(\QQ_{\theta} || \PP_r) - 2 \JSD(\PP_r, \QQ_{\theta}) \\
 	&= \KL(\QQ_{\theta} || \PP_r) - \KL(\PP_r || \QQ_{m}) - \KL(\QQ_{\theta} || \QQ_m)\\
 	&= 2 \KL(\QQ_{m} || \PP_r)
 	\end{align*}
 	Therefore the generator is trained by effectively optimizing the reverse KL between the mixture $\QQ_m$ and the real distribution $\PP_r$.
 	Hence, using that $\nabla_{\theta} q_{m}(x) = \nabla_{\theta} q_{\theta}(x)/2$
 	\begin{align*}
 	&\nabla_{\theta} 2 \KL(\QQ_{m} || \PP_r) 
 	\\& = \mathbb{E}_{\QQ_{\theta}} \left[\nabla_{\theta} \log(q_{\theta}) +
 	\log \left(\frac{ q_{m}(x)}{p_{r}(x)} \right) \nabla_{\theta} \log(q_{\theta})\right].
 	\end{align*}
 	
 	 	From \cite{arjovsky2017towards}, we know that
 	\begin{align*}
 	&\mathbb{E}_{ z \sim p(z) } [ \nabla_{\theta} \log\left(1- D^{*}( g_{\theta}(z))\right)|_{\theta = \theta_{0}} ] \\
 	&= 2 \left.\nabla_{\theta} \JSD(\QQ_{\theta}, \PP_r)\right|_{\theta = \theta_{0}}.
 	\end{align*}
 	Since
 	\[
 	\KL(\QQ_{\theta} || \PP_r) - 2 \JSD(\PP_r, \QQ_{\theta}) = 2 \KL(\QQ_{m} || \PP_r),
 	\]
 	we deduce from Proposition \ref{prp:negative_logD}
 	\begin{align*}
 	& 2 \nabla_{\theta} \JSD(\PP_r, \QQ_{\theta}) \\
 	&= \int \nabla_{\theta} q_{\theta}(x) \log \left( \frac{q_{\theta}(x) }{q_m(x)}\right)
 	dx.
 	\end{align*}
 \end{proof}

\textbf{Acknowledgement} 
We would like to thank 
Joey Bose for his technical support, 
Gavin Ding for his discussion,
and Hamidreza Saghir and Cathal Smyth for their edits and corrections.

\clearpage

\nocite{langley00}

\bibliography{paper}
\bibliographystyle{icml2017}

\end{document}